\documentclass[onecolumn,12pt]{IEEEtran}

\usepackage{subfig}
\usepackage{etoolbox}
\usepackage{tikz}
\usepackage{xr}
\usepackage{graphicx}
\usepackage{algorithm,algorithmic,color}
\usepackage{amsmath, amsfonts,amsthm,amssymb}
\usepackage{paralist}

\newcommand\1[1]{\mathbb{I}_{\left\{#1\right\}}}
\newtheorem{lemma}{Lemma}\newtheorem{theorem}{Theorem}\newtheorem{corollary}{Corollary}\newtheorem{conjecture}{Conjecture}
\newtheorem{proposition}{Proposition}
\newcommand{\post}[2]{\begin{center} \includegraphics[width=#2]{#1} \end{center} }
\newcommand \E[1]{\mathbb{E}[#1]}

\begin{document}
%

\title{Jointly Clustering Rows and Columns of Binary Matrices: Algorithms and Trade-offs}

\author{
    \IEEEauthorblockN{Jiaming Xu\IEEEauthorrefmark{1}, Rui Wu\IEEEauthorrefmark{1}, Kai Zhu\IEEEauthorrefmark{2}, Bruce Hajek\IEEEauthorrefmark{1}, R. Srikant\IEEEauthorrefmark{1}, Lei Ying\IEEEauthorrefmark{2}}\\
    \IEEEauthorblockA{\IEEEauthorrefmark{1}University of Illinois at Urbana-Champaign\\
    }
    \IEEEauthorblockA{\IEEEauthorrefmark{2}Arizona State University
    }
}

\maketitle
%
%

%
%
\begin{abstract}
In standard clustering problems, data points are represented
by vectors, and by stacking them together, one forms a data
matrix with row or column cluster structure.
In this paper, we consider a class of binary matrices, arising in many
applications, which exhibit both row and column cluster
structure, and our goal is to exactly recover the underlying
row and column clusters by observing only a small fraction
of noisy entries. We first derive a lower bound on the minimum number of observations needed for exact
cluster recovery.
Then, we propose three algorithms with different runtime and compare the number of observations
needed by them for successful cluster recovery. Our analytical results show smooth time-data trade-offs: one can gradually reduce the computational complexity when
increasingly more observations are available.


\end{abstract}

\section{Introduction}
Data matrices exhibiting both row and column cluster structure,
arise in many applications, such as collaborative filtering, gene expression analysis,
and text mining. For example, in recommender systems, a rating matrix can be formed with
rows corresponding to users and columns corresponding to items, and similar users and items form clusters.
In DNA microarrays, a gene expression matrix can be formed with rows corresponding to patients and columns corresponding to genes,
and similar patients and genes form clusters. Such row and column cluster structure is of great scientific interest and practical importance.
For instance, the user and movie cluster structure  is crucial for predicting user preferences and making accurate item recommendations \cite{York03}.
The patient and gene cluster structure reveals functional relations among genes and helps disease detection \cite{Madeira04,Jiang04}.
In practice, we usually only observe a very small fraction of entries in these data matrices,
possibly contaminated with noise, which obscures the intrinsic cluster structure.
For example, in Netflix movie dataset, about $99\% $ of movie ratings are missing and the observed ratings are noisy \cite{Volinsky09}.


In this paper, we study the problem of inferring hidden row and column cluster structure in binary data matrices from a few noisy observations.
We consider a simple model introduced in \cite{Aditya11,Dabeer12} for generating binary data matrix from underlying row and column clusters.
In the context of movie recommender systems, our model assumes that users and movies each form equal-sized clusters. Users in the same cluster
give the same rating to movies in the same cluster, where ratings are either $+1$ or $-1$ with $+1$ being ``like'' and $-1$ being ``dislike''. Each rating is
flipped independently with a fixed flipping probability less than $1/2$, modeling the noisy user behavior and the fact that users (movies) in the same cluster do not necessarily give (receive) identical ratings. Each rating is further erased independently with a erasure probability, modeling the fact that some ratings are not observed. Then, from the observed noisy ratings, we aim to {\it exactly} recover the underlying user and movie clusters, i.e., jointly cluster the rows and columns of the observed rating matrix.

The binary assumption on data matrices is of practical interest. Firstly, in many real datasets like Netflix dataset and DNA microarrays, estimation of entry values appears to be very unreliable, but the task of determining whether an entry is $+1$ or $-1$ can be done more reliably \cite{Dabeer12}. Secondly, in recommender systems like rating music on Pandora or rating posts on sites such as Facebook and MathOverflow, the user ratings are indeed binary \cite{Davenport12}.
The equal-sized assumption on cluster size is just for ease of presentation and can be relaxed to allow for different cluster sizes.

The hardness of our cluster recovery problem is governed by the erasure probability and cluster size. Intuitively, recovery becomes harder when the erasure probability increases, meaning fewer observations, and the cluster size decreases, meaning that clusters are harder to detect. The first goal of this paper is to understand when exact cluster recovery is possible or fundamentally impossible. Furthermore, our cluster recovery problem poses a computational challenge: An algorithm exhaustively searching over all the possible row and column cluster structures would have a time complexity exponentially increasing with the matrix dimension. The second goal of this paper is to understand how the computational complexity of our cluster recovery problem changes when increasingly more observations are available.


In this paper, our contributions are as follows. We first derive a lower bound on the minimum number of observations needed for exact cluster recovery as a function of matrix dimension and cluster size. Then we propose three algorithms with different runtimes and compare the number of observations
needed by them for successful cluster recovery.
\begin{compactitem}
\item The first algorithm directly searches for the optimal clustering of rows and columns separately; it is combinatorial in nature and takes exponential-time but achieves the best statistical performance among the three algorithms in the noiseless setting.
\item By noticing that the underlying true rating matrix is a specific type of low-rank matrix, the second algorithm recovers the clusters by solving a nuclear norm regularized convex optimization problem, which is a popular heuristic for low rank matrix completion problems; it takes polynomial-time but has less powerful statistical performance than the first algorithm.
\item The third algorithm applies spectral clustering to the rows and columns separately and then performs a joint clean-up step; it has lower computational complexity than the previous two algorithms, but less powerful statistical performance. We believe that this is the first such performance guarantee for exact cluster recovery, with a growing number of clusters, using spectral clustering.
\end{compactitem}
These algorithms are then compared with a simple nearest-neighbor clustering algorithm proposed in \cite{Dabeer12}.
Our analytical results show smooth time-data trade-offs: when increasingly more observations are available,
one can gradually reduce the computational complexity  by applying simpler algorithms while still achieving the desired performance. Such time-data trade-offs is of great practical interest for statistical learning problems involving large datasets \cite{chandrasekaran2013tradeoff}.


The rest of the paper is organized as follows. In Section~\ref{Sec:RelatedWork}, we discuss related work. In Section~\ref{Sec:Model}, we formally introduce our model and main results.
The lower bound is presented in Section~\ref{Sec:LowerBound}. The combinatorial method, convex method, spectral method are studied in Section~\ref{Sec:CombinatorialMethod}, Section~\ref{Sec:ConvexMethod} and Section~\ref{Sec:SpectralMethod}, respectively. The proofs are given in Section~\ref{Sec:Proofs}. The simulation results are presented in Section \ref{Sec: Sim}. Section \ref{Sec:Conclusion} concludes the paper with remarks.

\section{Related work}\label{Sec:RelatedWork}

In this section, we point out some connections of our model and results to prior work. There is a vast literature on clustering and we only focus on theoretical works with rigorous performance analysis. More detailed comparisons are provided after we present the theorems.

\subsection{Graph clustering}
Much of the prior work on graph clustering, as surveyed in \cite{Fortunato10}, focuses on graphs with a single node type, where nodes in the same cluster are more likely to have edges among them. A low-rank plus sparse matrix decomposition approach is proved to exactly recover the clusters with the best known performance guarantee in \cite{Chen12}. The same approach is used to recover the clusters from a partially observed graph in \cite{Chen13}. A spectral method for exact cluster recovery is proposed and analyzed in \cite{McSherry01} with the number of clusters {fixed}. More recently, \cite{Yu11} proved an upper bound on the number of nodes ``mis-clustered'' by a spectral clustering algorithm in the high dimensional setting with a growing number of clusters.

In contrast to the above works, in our model, we have a labeled bipartite graph with two types of nodes (rows and columns). Notice that there are no edges among nodes of the same type and cluster structure is defined for the two types separately. In this sense, our cluster recovery problem can be viewed as a natural generalization of graph clustering problem to labeled bipartite graphs.  In fact, our second algorithm via convex programming is inspired by the work \cite{Chen12, Chen13}.

A model similar to ours but with a fixed number of clusters has been considered in \cite{mastom11}, where the spectral method plus majority voting is shown to {\it approximately} predict the rating matrix. However, our third algorithm via spectral method is shown to achieve exact cluster and rating matrix recovery with a growing number of clusters. This is the first theoretical
result on spectral method for exact cluster recovery in
with a growing number of clusters to our knowledge.

\subsection{Biclustering}
Biclustering \cite{Hartigan72,ExpressionData00,Madeira04, Busygin08} tries to find sub-matrices (which may overlap) with particular patterns in a data matrix. Many of the proposed algorithms are based on heuristic searches without provable performance guarantees. Our cluster recovery problem can be viewed as a special case where the data matrix consists of non-overlapping sub-matrices with constant binary entries, and our paper provides a thorough study of this special biclustering problem. Recently, there is a line of work studying another special case of biclustering problem, which tries to detect a single small submatrix with elevated mean in a large fully observed noisy matrix \cite{Kolar2011}. Interesting statistical and computational trade-offs are summarized in \cite{balakrishnan2011}.

\subsection{Low-rank matrix completion}
Under our model, the underlying true data matrix is a specific type of low-rank matrix. If we recover the true data matrix, we immediately get the user (or movie) clusters by assigning the identical rows (or columns) of the matrix to the same cluster. In the noiseless setting with no flipping, the nuclear norm minimization approach \cite{Candes10,Candes12,Recht11} can be directly applied to recover the true data matrix and further recover the row and column clusters.  Alternate minimization is another popular and empirically successful approach for low-matrix completion \cite{Volinsky09}. However, it is harder to analyze and the performance guarantee is weaker than nuclear norm minimization \cite{Sanghavi12}. In the low noise setting with the flipping probability restricting to be a small constant, the low-rank plus sparse matrix decomposition approach \cite{Sanghavi11,Candes11,ChenIT13} can be applied to exactly recover data matrix and further recover the row and column clusters.

The performance guarantee for our second algorithm via convex programming is better than these previous approaches and it allows the flipping probability to be any constant less than $1/2$. Moreover, our proof turns out to be much simpler. The recovery of our true data matrix can also be viewed as a specific type of one-bit matrix completion problem recently studied in \cite{Davenport12}. However, \cite{Davenport12} only focuses on approximate matrix recovery and the results there cannot be used to recover row and column clusters.

\begin{table*}[t!]
    \renewcommand{\captionfont}{\small}               
    \renewcommand{\captionlabelfont}{\small}          
    \small                                       
    \centering
    \renewcommand{\arraystretch}{1.6} 
    \begin{tabular}[c]{|c|l|l|c|c|}
      \hline
      & parameter regime & $\#$ of observations & runtime  & remark\\
      \hline
      lower bound & $nK^2(1-\epsilon)^2 = O(1) $ & $m =O \left( \frac{n^{1.5}}{K} \right) $ & & \\
      \hline
      combinatorial method & $nK(1-\epsilon)^2 =\Omega(\log n) $ & $m = \Omega( \frac{n^{1.5} \sqrt{\log n} }{ \sqrt{K } } )$  & exponential  &  assuming noiseless, i.e., $p = 0$ \\
      \hline
      convex method & $K(1-\epsilon) = \Omega( \log n) $ & $m = \Omega(\frac{n^2 \log n}{K} )$ & polynomial & assuming Conjecture~\ref{result:conjecture} holds \\
      \hline
      spectral method & $K(1-\epsilon) = \Omega( r\log^2 n)$ & $m = \Omega( \frac{n^3 \log^2 n}{K^2} )$ & $O(n^3)$ & \\
      \hline
      nearest-neighbor clustering &  $n(1-\epsilon)^2 =\Omega(\log n)$ & $m = \Omega( n^{1.5} \sqrt{\log n}) $ & $O(mr)$ & \\
      \hline
  \end{tabular}
  \caption{Main results: comparison of a lower bound and four clustering algorithms. }
  \label{table:comparison}
\end{table*}


\section{Model and Main Results} \label{Sec:Model}
In this section, we formally state our model and main results.
\subsection{Model}
Our model is described in the context of movie recommender systems, but it is applicable to other systems with binary data matrices having row and column cluster structure.
Consider a movie recommender system with $n$ users and $n$ movies. Let $R$ be the rating matrix of size $n\times n$ where $R_{ij}$ is the rating user $i$ gives to movie $j$. Assume both users and movies form $r$ clusters of size $K = n/r$. Users in the same cluster give the same rating to movies in the same cluster. The set of ratings corresponding to a user cluster and a movie cluster is called a block. Let $B$ be the \emph{block rating matrix} of size $r\times r$ where $B_{kl}$ is the \emph{block rating} user cluster $k$ gives to movie cluster $l$. Then the rating $R_{ij} = B_{kl}$ if user $i$ is in user cluster $k$ and movie $j$ is in movie cluster $l$. Further assume that entries of $B$ are independent random variables which are $+1$ or $-1$ with equal probability. Thus, we can imagine the rating matrix as a block-constant matrix with all the entries in each block being either $+1$ or $-1$.
Observe that if $r$ is a fixed constant, then users from two different clusters have the same ratings for all movies with some positive probability, in which case it is impossible to differentiate between these two clusters. To avoid such situations, assume $r$ is at least $\Omega(\log n)$.

Suppose each entry of $R$ goes through an independent binary symmetric channel with flipping probability $p<1/2$, representing noisy user behavior, and an independent erasure channel with erasure probability $\epsilon$, modeling the fact that some entries are not observed. The expected number of observed ratings is $m = n^2(1-\epsilon)$. We assume that $p$ is a constant throughout the paper and $\epsilon$ could converge to $1$ as $n\to \infty$. Let $R'$ denote the output of the binary symmetric channel and $\Omega$ denote the set of non-erased entries. Let $\widehat{R}_{ij} = R'_{ij}$ if $(i, j) \in \Omega$ and $\widehat{R}_{ij}=0$ otherwise. The goal is to exactly recover the row and column clusters from the observation $\widehat{R}$.


\subsection{Main Results}
The main results are summarized in Table~\ref{table:comparison}. Note that these results do not explicitly depend on $p$. In fact, as $p$ is assumed to be a constant strictly less than $1/2$, it affects the results by constant factors.


The parameter regime where exact cluster recovery is fundamentally impossible for any algorithm is proved in Section~\ref{Sec:LowerBound}. The combinatorial method, convex method and spectral method are studied in Section~\ref{Sec:CombinatorialMethod}, Section~\ref{Sec:ConvexMethod} and Section~\ref{Sec:SpectralMethod}, respectively. We only analyze the combinatorial method in the noiseless case where $p = 0$, but we believe similar result is true for the noisy case as well. The parameter regime in which the convex method succeeds is obtained by assuming that a technical conjecture holds, which is justified through extensive simulation. The parameter regime in which the spectral method succeeds is obtained for the first time for exact cluster recovery with a growing number of clusters. The nearest-neighbor clustering algorithm was proposed in \cite{Dabeer12}. It clusters the users by finding the $K-1$ most similar neighbors for each user. The similarity between user $i$ and $i'$ is measured by the number of movies with the same observed rating, i.e.,
\begin{align*}
s_{i i'} = \sum_{j = 1}^n \1{\widehat{R}_{i j} \neq 0 } \1{\widehat{R}_{i'j} \neq 0 } \1{ \widehat{R}_{ij}=\widehat{R}_{i'j} },
\end{align*}
where $\1{\cdot}$ is an indicator function. Movies are clustered similarly. It is shown in \cite{Dabeer12} that the nearest-neighbor clustering algorithm exactly recovers user and movie clusters when $n(1-\epsilon)^2>C\log n$ for a constant $C$.

The number of observations needed for successful cluster recovery can be derived from the corresponding parameter regime using the identity $m=n^2(1-\epsilon)$ as shown in Table~\ref{table:comparison}. For better illustration, we visualize our results in Figure~\ref{fig:summary}. In particular, we take $\log (m/n)$ as $x$-axis and $\log K$ as $y$-axis and normalize both axes by $\log n$. Since exact cluster recovery becomes easy when the number of observations $m$ and cluster size $K$ increase, we expect that exact cluster recovery is easy near $(1, 1)$ and hard near $(0, 0)$.

\begin{figure}[h!]
\begin{center}
\scalebox{1}{\begin{tikzpicture}[scale = 2.5, font = \small, thick]
\draw[->] (0, 0) node [below left] {$O$}-- (2.3, 0) node [right]{\large $\frac{\log (m/n)}{\log n}$};
\draw[->] (0, 0) -- (0, 2.3) node [left]{\large $\frac{\log K}{\log n}$};
\draw (2, 0) -- (2, 2);
\draw (0, 2) -- (2, 2);
\node [left] at (0, 2) {$1$};
\node [below] at (2, 0) {$1$};
\node [left] at (0,1) {$1/2$};
\node [below] at (1, 0) {$1/2$};
\node [below right] at (2, 1) {$1/2$};
\draw[color = green,  text = black, line width = 1] (1, 0) node [above right] {$A$} -- (1, 2) node [above] {$D$};
\draw[color = red, text = black,line width = 1 ] (0, 2) -- (2, 0) node [above right] {$B$};
\path[fill = black!20] (0, 0) -- (0, 1) -- (1, 0) -- cycle;
\draw[color = yellow, text = black, line width = 1, line cap = round, shorten <=1, shorten >= 1] (1,0) -- (0, 2) node [above right] {$E$};
\draw[color=blue, text=black,line width=1] (2,1) node [above right]{$C$} --(0,2);
\end{tikzpicture}}
\end{center}
\caption{Summary of results in terms of number of observations $m$ and cluster size $K$. The lower bound states that it is impossible for any algorithm to reliably recover the clusters exactly in the shaded regime (grey). The combinatorial method, the convex method, the spectral method and the nearest-neighbor clustering algorithm succeed in the regime to the right of lines $AE$ (yellow), $BE$ (red), $CE$ (blue) and $AD$ (green), respectively. }
\label{fig:summary}
\end{figure}
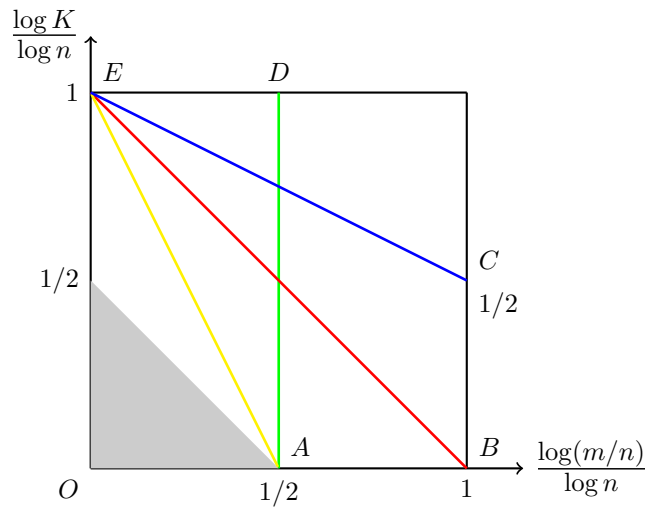

From Figure~\ref{fig:summary}, we can observe interesting trade-offs between algorithmic runtime and statistical performance. In terms of the runtime, the combinatorial method is exponential, while the other three algorithms are polynomial. In particular, the convex method can be casted as a semidefinite programming and solved in polynomial-time. For the spectral method, the most computationally expensive step is the singular value decomposition of the observed data matrix which can always be done in time $O(n^3)$ and more efficiently when the observed data matrix is sparse. It is not hard to see that the time complexity for the nearest-neighbor clustering algorithm is $O(n^2r)$ and more careful analysis reveals that its time complexity is $O(mr)$. On the other hand, in terms of statistical performance, the combinatorial method needs strictly fewer observations than the other three algorithms when there is no noise, and the convex method  always needs fewer observations than the spectral method. It is somewhat surprising to see that the simple nearest-neighbor clustering algorithm needs fewer observations than the more sophisticated convex method when the cluster size $K$ is $O(\sqrt{n})$.

In summary, we see that when more observations available, one can apply algorithms with less runtime while still achieving exact cluster recovery. For example, consider the noiseless case with cluster size $K=n^{0.8}$, the number of observations {\em per user} required for cluster recovery by the combinatorial method, convex method, spectral method and nearest-neighbor clustering algorithm are $\Omega(n^{0.1})$, $\Omega(n^{0.2})$, $\Omega(n^{0.4})$ and $\Omega(n^{0.5})$, respectively. Therefore, when the number of observations per user increases from $\Omega(n^{0.1})$ to $\Omega(n^{0.5})$, one can gradually reduces the computational complexity from exponential-time to polynomial-time as low as $O(n^{1.7})$.


The main results in this paper can be easily extended to the more
general case with $n_1$ rows and $n_2=\Theta(n_1)$ columns and $r_1$ row
clusters and $r_2=\Theta(r_1)$ column clusters. The sizes of different
clusters could vary as long as they are of the same
order. Likewise, the flipping probability $p$ and the erasure probability
$\epsilon$ could also vary for different entries of the data matrix
as long as they are of the same order.
Due to space constraints, such generalizations are omitted in this paper.

\subsection{Notations}
A variety of norms on matrices will be used. The spectral norm of a matrix $X$ is denoted by $\|X\|$, which is equal to the largest singular value. Let $\langle X, Y \rangle= \text{Tr}(X^\top Y) $ denote the inner product between two matrices. The nuclear norm is denoted by $\|X\|_\ast$ which is equal to the sum of singular values and is a convex function of $X$. Let $\|X\|_1=\sum_{i,j}|X_{ij}|$ denote the $l_1$ norm and $\|X\|_\infty=\max_{i,j} |X_{ij}|$ denote the $l_\infty$ norm. Let $X=\sum_{t=1}^n \sigma_t u_t v_t^\top$ denotes the singular value decomposition $X \in \mathbb{R}^{n \times n}$ such that $\sigma_1 \ge \cdots \ge \sigma_n$. The best rank $r$ approximation of $X$ is defined as ${P}_r(X)=\sum_{t=1}^r \sigma_t u_t v_t^\top$. For vectors, let $\langle x, y \rangle$ denote the inner product between two vectors and the only norm that will be used is the usual $l_2$ norm, denoted as $\|x\|_2$.

Throughout the paper, we say that an event occurs ``a.a.s.'' or ``asymptotically almost surely'' when it occurs with a probability which tends to one as $n$ goes to infinity.

\section{Lower Bound} \label{Sec:LowerBound}
In this section, we derive a lower bound for any algorithm to
reliably recover the user and movie clusters. The lower bound is constructed by considering a genie-aided scenario
where the set of flipped entries is revealed as side information, which is equivalent to saying that we are in the noiseless
setting with $p = 0$. Hence, the true rating matrix $R$ agrees with $\widehat{R}$ on all non-erased entries. We construct another
rating matrix $\tilde{R}$ with the same movie cluster structure as $R$ but different user cluster structure by swapping two users in
two different user clusters. We show that if $nK^2(1-\epsilon)^2=O(1)$, then $\tilde{R}$ agrees with $\widehat{R}$ on all non-erased entries with positive probability, which implies that no algorithm can reliably distinguish between $R$ and $\widehat{R}$ and thus recover user clusters.
\begin{theorem} \label{ThmLowerbound}
Fix $0< \delta<1$. If $nK^2 (1-\epsilon)^2 <\delta$, then with probability at least $1-\delta$, it is impossible for any algorithms to  recover the user clusters or movie clusters.
\end{theorem}
Intuitively, Theorem \ref{ThmLowerbound} says that when the erasure probability is high and the cluster size is small that $nK^2(1-\epsilon)^2=O(1)$, the observed rating matrix $\widehat{R}$ does not carry enough information to distinguish between different possible cluster structures.

\section{Combinatorial Method} \label{Sec:CombinatorialMethod}
In this section, we study a combinatorial method which clusters users or movies by searching for a partition with the least total number of ``disagreements''. We describe the method in Algorithm \ref{alg:combinatorialmethod} for clustering users only. Movies are clustered similarly. The number of disagreements $D_{ii'}$ between a pair of users $i,i'$ is defined as the number of movies satisfying that: The two ratings given by users $i,i'$ are both observed and the observed two ratings are different. In particular, if for every movie, the two ratings given by users $i, i'$ are not observed simultaneously, then $D_{ii'}=0$.

\begin{algorithm}
\caption{Combinatorial Method}\label{alg:combinatorialmethod}
\begin{algorithmic}[1]
\STATE For each pair of users $i,i'$, compute the number of disagreements $D_{ii'}$ between them.
\STATE For each partition of users into $r$ clusters of equal size $K$, compute its total number of disagreements defined as
$$\sum_{i, i' \text{ in the same cluster}}D_{ii'} $$
\STATE Output a partition which has the least total number of disagreements.
\end{algorithmic}
\end{algorithm}

The idea of Algorithm~\ref{alg:combinatorialmethod} is to reduce the problem of clustering both users and movies to a standard user clustering problem without movie cluster structure. In fact, this algorithm looks for the optimal partition of the users which has the minimum total in-cluster distance, where the distance between two users is measured by the number of disagreements between them. The following theorem shows that such simple reduction does {\em not} achieve the lower bound given in Theorem~\ref{ThmLowerbound}. The optimal algorithm for our cluster recovery problem might need to explicitly make use of both user and movie cluster structures.

\begin{theorem}\label{result:combinatorial_lowerbound}
If $nK(1-\epsilon)^2 \le \frac{1}{4}$, then with probability at least $3/4$, Algorithm \ref{alg:combinatorialmethod} cannot recover user and movie clusters.
\end{theorem}

Next we show that the above necessary condition for the combinatorial method is also sufficient up to a logarithmic factor when there is no noise, i.e., $p = 0$.
We suspect that the theorem holds for the noisy setting as well, but we have not yet been able to prove this.

\begin{theorem}\label{result:noiseless_clustering}
If $p=0$ and $ nK(1-\epsilon)^2 >C \log n$ for some constant $C$, then a.a.s. Algorithm \ref{alg:combinatorialmethod} exactly recovers user and movie clusters.
\end{theorem}

This theorem is proved by considering a conceptually simpler greedy algorithm that does not need to know $K$. After computing the number of disagreements for every pair of users, we search for a largest set of users which have no disagreement between each other, and assign them to a new cluster. We then remove these users and repeat the searching process until there is no user left. In the noiseless setting, the $K$ users from the same true cluster have no disagreement between each other. Therefore, it is sufficient to show that, for any set of $K$ users consisting of users from more than one cluster, they have more than one disagreement with high probability under our assumption.
%

\section{Convex Method} \label{Sec:ConvexMethod}
In this section, we show that the rating matrix $R$ can be exactly recovered by a convex program, which is a relaxation of the maximum likelihood (ML) estimation. When $R$ is known, we immediately get the user (or movie) clusters by assigning the identical rows (or columns) of $R$ to the same cluster.

Let $\mathcal{Y}$ denote the set of binary block-constant rating matrix with $r^2$ blocks of equal size. As the flipping probability $p<1/2$, Maximum Likelihood (ML) estimation of $R$ is equivalent to finding a $Y \in \mathcal{Y}$ which best matches the observation $\widehat{R}$:
\begin{align}
\max_{Y}  & \; \sum_{i,j} \widehat{R}_{ij} Y_{ij} \nonumber \\
\text{s.t.	} & \; Y \in \mathcal{Y}. \label{MLE}
\end{align}
Since $|\mathcal{Y}|= \Omega(e^{n})$, solving~\eqref{MLE} via exhaustive search takes exponential-time. Observe that $Y \in \mathcal{Y}$ implies that $Y$ is of rank at most $r$. Therefore, a natural relaxation of the constraint that $Y \in \mathcal{Y}$ is to replace it with a rank constraint on $Y$, which gives the following problem:
\begin{align*}
\max_{Y}  & \;\sum_{i,j} \widehat{R}_{ij} Y_{ij} \nonumber \\
\text{s.t.	} & \; \text{rank} (Y) \leq r,  \nonumber
 \;Y_{ij} \in \{1, -1 \}. 
\end{align*}
Further by relaxing the integer constraint and replacing the rank constraint with the nuclear norm regularization, which is a standard technique for low-rank matrix completion, we get the desired convex program:
\begin{align}
\max_{Y}  & \;\sum_{i,j} \widehat{R}_{ij} Y_{ij} - \lambda \| Y \|_\ast \nonumber \\
\text{s.t.	}  & \; Y_{ij}  \in [-1,1].  \label{LowRank}
\end{align}
The clustering algorithm based on the above convex program is given in Algorithm \ref{alg:convexmethod}
\begin{algorithm}
\caption{Convex Method}\label{alg:convexmethod}
\begin{algorithmic}[1]
\STATE (Rating matrix estimation) Solve for $\widehat{Y}$ the convex program~\eqref{LowRank}.
\STATE (Cluster estimation) Assign identical rows (columns) of $\widehat{Y}$ to the same cluster.
\end{algorithmic}
\end{algorithm}

The convex program~\eqref{LowRank} can be casted as a semidefinite program and solved in polynomial-time. Thus, Algorithm \ref{alg:convexmethod} takes polynomial-time. Our performance guarantee for Algorithm \ref{alg:convexmethod} is stated in terms of the incoherence parameter $\mu$ defined below. Since the rating matrix $R$ has rank $r$, the singular vector decomposition (SVD) is $R = U\Sigma V^\top $, where $U, V\in\mathbb{R}^{n\times r}$ are matrices with orthonormal columns and $\Sigma\in\mathbb{R}^{r\times r}$ is a diagonal matrix with nonnegative entries.
Define $\mu>0$ such that $\| UV^\top \|_{\infty} \leq \mu \sqrt{r}/n$.

Denote the SVD of the block rating matrix $B$ by $B = U_B\Sigma_B V_B^\top$. The next lemma shows that
 \begin{align}
 \|UV^\top\|_{\infty} = \|U_BV_B^\top\|_{\infty}/K, \label{Eq:Incoherence}
 \end{align}
and thus $\mu$ is upper bounded by $\sqrt{r}$.
\begin{lemma}\label{result:incoherence_loose}
  $\mu\leq \sqrt{r}$.
\end{lemma}

The following theorem provides a sufficient condition under which Algorithm \ref{alg:convexmethod} exactly recovers the rating matrix and thus the row and column clusters as well.
\begin{theorem}\label{result:convex_relaxation}
  If $n(1-\epsilon)\geq C'\log^2 n$ for some constant $C'$, and
  \begin{align}
    m >C nr  \max\{\log n, \mu^2\}, \label{EqCondition}
  \end{align}
  where $C$ is a constant and $\mu$ is the incoherence parameter for $R$, then a.a.s. the rating matrix $R$ is the unique maximizer to the convex program \eqref{LowRank} with $\lambda=3 \sqrt{(1-\epsilon)n}$.
\end{theorem}
 Note that Algorithm \ref{alg:convexmethod} is easy to implement as $\lambda$ only depends on the erasure probability $\epsilon$, which can be reliably estimated from $\widehat{R}$. Moreover, the particular choice of $\lambda$ in the theorem is just to simplify notations. It is straightforward to generalize our proof to show that the above theorem holds with $\lambda=C_1\sqrt{(1-\epsilon)n}$ for any constant $C_1\ge 3$.

Using Lemma~\ref{result:incoherence_loose}, we immediately conclude from the above theorem that the convex program succeeds when $m>Cn r^2$ for some constant $C$. However, based on extensive simulation in Fig~\ref{FigIncoherenceNew}, we conjecture that the following result is true.
\begin{conjecture} \label{result:conjecture}
$\mu=\Theta(\sqrt{\log r})$ a.a.s.
\end{conjecture}
By~\eqref{Eq:Incoherence}, Conjecture \ref{result:conjecture} is equivalent to $\|U_BV_B^\top\|_{\infty}=\Theta(\sqrt{\frac{\log r}{r} })$. For a fixed $r$, we simulate $1000$ independent trials of $B$, pick the largest value of $\|U_BV_B^\top\|_{\infty}$, scale it by dividing $\sqrt{\log r/r}$, and get the plot in Fig~\ref{FigIncoherenceNew}.
\begin{figure}
\centering
\post{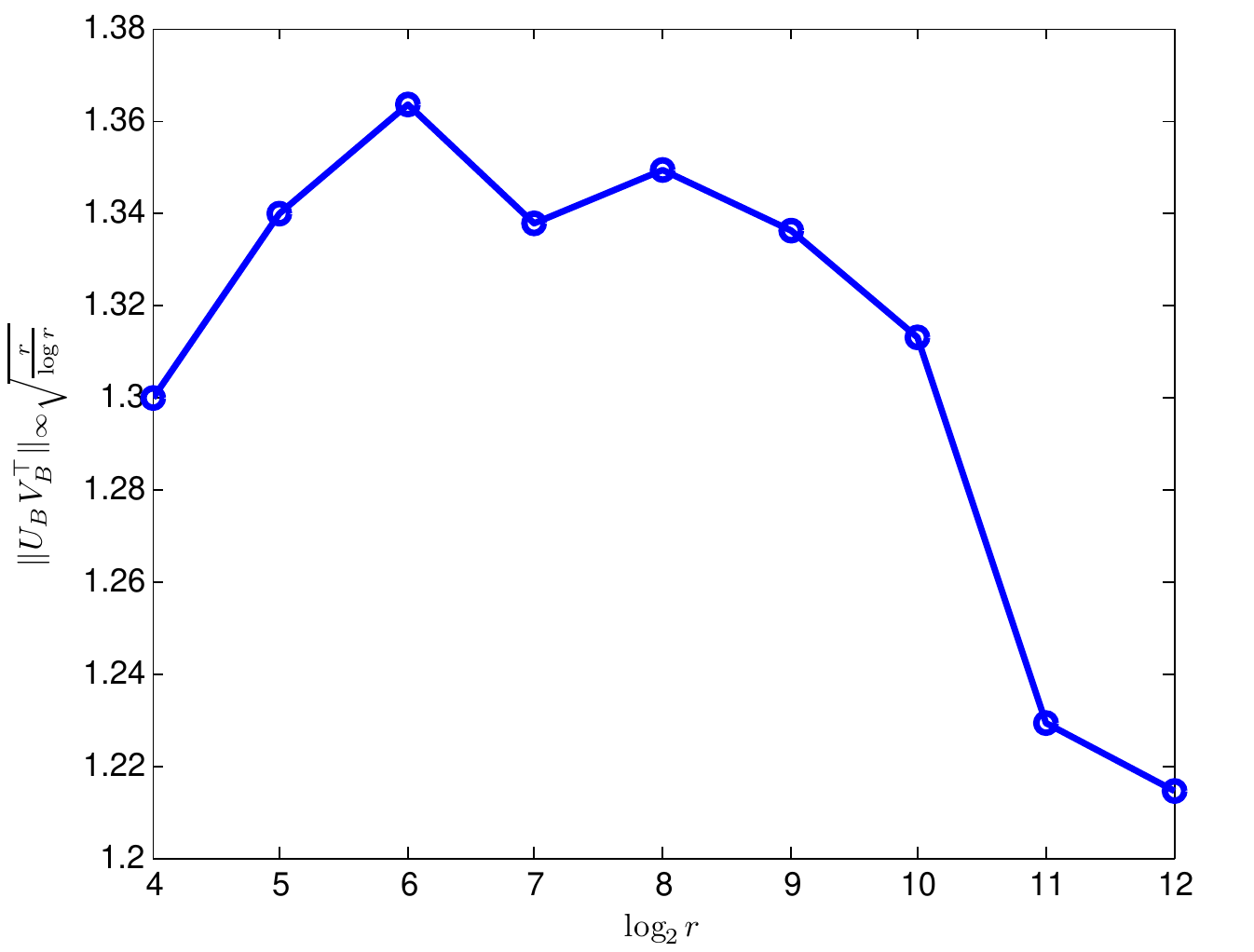}{3.2in}
\centering
\caption{Simulation result supporting Conjecture \ref{result:conjecture}. The conjecture is equivalent to $\|U_BV_B^\top\|_{\infty}=\Theta(\sqrt{\frac{\log r}{r} })$.}
\label{FigIncoherenceNew}
\end{figure}

Assuming this conjecture holds, Theorem~\ref{result:convex_relaxation} implies that $$m>Cn r\log n$$ for some constant $C$ is sufficient to recover the rating matrix, which is better than the previous condition by a factor of $r$. We do not have a proof for the conjecture at this time.

{\bf Comparison to previous work} \enskip In the noiseless setting with $p=0$, the nuclear norm minimization approach \cite{Candes10,Candes12,Recht11} can be directly applied to recover data matrix and further recover the row and column clusters. It is shown in \cite{Recht11} that the nuclear norm minimization approach exactly recovers the matrix with high probability if $m=\Omega(\mu^2 nr \log^2 n)$. The performance guarantee for Algorithm \ref{alg:convexmethod} given in \eqref{EqCondition} is better by at least a factor of  $\log n$.  Theorem \ref{result:noiseless_clustering} shows that the combinatorial method exactly recovers the row and column clusters if $m= \Omega(n r^{1/2} \log^{1/2}n)$, which is substantially better than the two previous conditions by at least a factor of $r^{1/2}$. This suggests that a large performance gap might exist between exponential-time algorithms and polynomial-time algorithms. Similar performance gap due to computational complexity constraint has also been observed in other inference problems like Sparse PCA \cite{berthet2012sparsePCA,berthet2013lowerSparsePCA,Vilenchik13} and sparse submatrix detection \cite{Kolar2011,balakrishnan2011,MaWu13}.

In the low noise setting with $p$ restricting to be a small constant, the low-rank plus sparse matrix decomposition approach \cite{Sanghavi11,Candes11,ChenIT13} can be applied to exactly recover data matrix and further recover the row and column clusters. It is shown in \cite{ChenIT13} that a weighted nuclear norm and $l_1$ norm minimization succeeds with high probability if $m=\Omega(\rho_r \mu^2 nr \log^6 n)$ and $p\le \rho_s$  for two constants $\rho_r$ and $\rho_s$. The performance guarantee for Algorithm \ref{alg:convexmethod} given in \eqref{EqCondition} is better by several $\log n$ factors and we allow the fraction of noisy entries $p$ to be any constant less than $1/2$. Moreover, our proof turns out to be much simpler.

\section{Spectral Method} \label{Sec:SpectralMethod}

\begin{algorithm}
\caption{Spectral Method}\label{alg:spectralmethod}
\begin{algorithmic}[1]
\STATE (Producing two
subsets, $\Omega_1$ and $\Omega_2,$ of $\Omega$ via randomly sub-sampling $\Omega$)
Let $\delta=\frac{1-\epsilon}{4},$  and independently assign
each element of $\Omega$ only to $\Omega_1$ with probability $\frac{1}{2}-\delta$,
only to $\Omega_2$ with probability $\frac{1}{2}-\delta$,
to both $\Omega_1$ and $\Omega_2$ with probability $\delta$, and to
neither $\Omega_1$ nor $\Omega_2$ with probability $\delta.$
Let $\widehat{R}^{(1)}_{i,j}=\widehat{R}_{i,j}\1{ (i,j)\in \Omega_1 }
$
and
$\widehat{R}^{(2)}_{i,j}=\widehat{R}_{i,j}\1{(i,j)\in \Omega_2}$
for $i,j \in \{1,\ldots ,n\}.$

\STATE  (Approximate clustering)  Let $P_r(\widehat{R}^{(1)})$ denote the
rank $r$ approximation of $\widehat{R}^{(1)}$ and let $x_i$ denote the
$i$-th row of $P_r(\widehat{R}^{(1)}).$   Construct user clusters
$\widehat{C}_1, \ldots  , \widehat{C}_r$ sequentially as follows.
For $1\leq k \leq r,$  after $\widehat{C}_1,\ldots , \widehat{C}_{k-1}$
have been selected, choose an initial user not in the first $k-1$ clusters, uniformly
at random, and let $\widehat{C}_k= \{ i' : ||x_i - x_{i'}|| \leq \tau \}.$
(The threshold $\tau$ is specified below.)
Assign each remaining unclustered user to a cluster arbitrarily.
Similarly, construct movie clusters $\widehat{D}_1, \ldots , \widehat{D}_r$
based on the columns of $P_r(\widehat{R}^{(1)}).$

\STATE (Block rating estimation by majority voting)  For $k, l \in \{ 1, \ldots , r\},$ let $\widehat{V}_{kl}=\sum_{i \in \widehat{C}_k } \sum_{j \in \widehat{D}_l } \widehat{R}^{(2)}_{ij}$ be the total vote that user cluster $\widehat{C}_k$ gives to movie cluster $\widehat{D}_l$.
If $\widehat{V}_{kl} \ge 0$, let $\widehat{B}_{kl}=1$; otherwise, let $\widehat{B}_{kl}=-1$.

\STATE (Reclustering by assigning users and movies to nearest centers) Recluster users as follows. For $k\in \{1, \dots, r\}$, define
center $\mu_k$ for user cluster $\widehat{C}_k$ as $\mu_{kj}=\widehat{B}_{k l}$ if movie $j\in \widehat{D}_l$ for all $j$. Assign user $i$ to cluster $k$ if $\langle \widehat{R}^{(2)}_{i,\cdot}, \mu_k \rangle \ge
\langle \widehat{R}^{(2)}_{i,\cdot}, \mu_{k'} \rangle$ for all $k' \neq k$. Recluster movies similarly.
\end{algorithmic}
\end{algorithm}

In this section, we study a polynomial-time clustering algorithm based on the
spectral projection of the observed rating matrix $\widehat{R}$. The description is given in Algorithm \ref{alg:spectralmethod}.

Step 1 of the algorithm produces two subsets, $\Omega_1$ and $\Omega_2,$ of $\Omega$ such that: $\mathrm{1})$ for $i\in \{1,2\},$ each rating is observed in $\Omega_i$ with probability $\frac{1-\epsilon}{2},$ independently of other elements; and $\mathrm{2})$ $\Omega_1$ is independent of $\Omega_2.$ The purpose of Step 1 is to remove dependency between Step 2 and Steps 3, 4 in our proof. In particular, to establish our theoretical results, we identify the initial clustering of users and movies using $\Omega_1$, and then majority voting and reclustering are done using $\Omega_2$. In practice, one can simply use the same set of observations, i.e., $\Omega_1 = \Omega_2 = \Omega$.

The following theorem shows that the spectral method exactly recovers the user and movie clusters under a condition stronger than \eqref{EqCondition}. In particular, we show that Step 3 exactly recovers the block rating matrix $B$ and Step 4 cleans up clustering errors made in Step 2.
\begin{theorem} \label{ThmSpectralMethod}
If
\begin{align}
n  (1-\epsilon )> C  r^2 \log^2 n, \label{EqSpectralCondition}
\end{align}
for a positive constant $C$, then Algorithm \ref{alg:spectralmethod}  with $\tau=12 (1-\epsilon)^{1/2} r \log n $ a.a.s. exactly recovers user and movie clusters, and the rating matrix $R$.
\end{theorem}

Algorithm \ref{alg:spectralmethod} is also easy to implement as $\tau$ only depends on parameters $\epsilon$ and $r$. As mentioned before, the erasure probability $\epsilon$ can be reliably estimated from
 $\widehat{R}$ using empirical statistics. The number of clusters $r$ can be reliably estimated by searching for the largest eigen-gap in the spectrum of $\widehat{R}$ (See Algorithm 2 and Theorem 3 in \cite{Chen12} for justification). We further note that the threshold $\tau$ used in the theorem can be replaced by $C_1 (1-\epsilon)^{1/2} r \log n$ for any constant $C_1\ge 12$.

{\bf Comparison to previous work} \enskip
Variants of spectral method are widely used for clustering nodes in a graph. Step 2 of Algorithm \ref{alg:spectralmethod} for approximate clustering has been previously proposed and it is analyzed in \cite{Kannan09}. In \cite{McSherry01}, an adaptation of Step 1 is shown to exactly recover a fixed number of clusters under the planted partition model. More recently, \cite{Yu11} proves an upper bound on the number of nodes ``mis-clustered'' by spectral method under the stochastic block model with a growing number of clusters.

Compared to previous work, the main novelty of Algorithm \ref{alg:spectralmethod} is Steps 1, 3, and 4 which allow for exact cluster recovery even with a growing number of clusters. To our knowledge, Theorem \ref{ThmSpectralMethod} provides the {first theoretical result on spectral method for exact cluster recovery with a growing number of clusters}.

\section{Proofs} \label{Sec:Proofs}
\subsection{Proof of Theorem \ref{ThmLowerbound}}
Without loss of generality, suppose that users $1, 3, \ldots, 2K-1$ are in cluster $1$ and users $2,4, \ldots, 2K$ are in cluster $2$. We construct a block-constant matrix with the same movie cluster structure as $R$ but a different user cluster structure. In particular, under $\tilde{R}$, user $1$ forms a new cluster with users $2i, i=2,\ldots, K$ and user $2$ forms a new cluster with users $2i-1, i=2, \ldots,K$.

Let $i$-th row of $\tilde{R}$ be identical to the $i$-th row of $R$ for all $i > 2K$. Consider all movies $j$ in movie cluster $l$. If the ratings of user $1$ to movies in movie cluster $l$ are all erased, then let $\tilde{R}_{1 j} = R_{2j}$ and $\tilde{R}_{ij}=R_{2j}$ for $i=4, 6, \ldots,2K$; otherwise let $\tilde{R}_{1j}=R_{1j}$ and $\tilde{R}_{ij}=R_{1j}$ for $i =4, 6, \ldots, 2K$. If the ratings of user $2$ to movies in movie cluster $l$ are all erased, then let $\tilde{R}_{2 j}= R_{1 j}$ and $\tilde{R}_{ij}=R_{1j}$ for $i=3, 5,\ldots,2K-1$; otherwise let $\tilde{R}_{2j}=R_{2j}$ and $\tilde{R}_{ij}=R_{2j}$ for $i=3, 5, \ldots, 2K-1$. From the above procedure, it follows that the first row of $\tilde{R}$ is identical to the $(2i)$-th row of $\tilde{R}$ for all $i =2, \ldots, K$, and the second row of $\tilde{R}$ is identical the $(2i-1)$-th row of $\tilde{R}$ for all $i=2, \ldots, K$.

We show that $\tilde{R}$ agrees with $\widehat{R}$ on all non-erased entries. We say that movie cluster $l$ is conflicting between user $1$ and user cluster $2$ if (1) user cluster $1$ and $2$ have different block rating on movie cluster $l$; and (2) the ratings of user $1$ to movies in movie cluster $l$ are not all erased; and (3) the block corresponding to user cluster $2$ and movie cluster $l$ is not totally erased. Therefore, the probability that  movie cluster $l$ is conflicting between user $1$ and user cluster $2$ equals to $\frac{1}{2} (1-\epsilon^{k^2}) (1-\epsilon^{k})$. By the union bound,
\begin{align*}
& \mathbb{P} \{ \exists \text{conflicting movie cluster between user $1$ and cluster $2$}\} \\
& \le \frac{r}{2} (1-\epsilon^{k^2}) (1-\epsilon^{k}) \le \frac{r}{2} K^3(1-\epsilon)^2 \le \delta/2,
\end{align*}
where the third inequality follows because $(1-x)^{a} \ge 1-ax$ for $a \ge 1$ and $x \ge 0$ and the last inequality follows from the assumption. Similarly, the probability that there exists a conflicting movie cluster between user $2$ and cluster $1$ is also upper bounded by $\delta/2$. Hence, with probability at least $1-\delta$, there is no conflicting movie cluster between user $1$ and cluster $2$ as well as between user $2$ and cluster $1$, and thus $\tilde{R}$ agrees with $\widehat{R}$ on all non-erased entries.

\subsection{Proof of Theorem \ref{result:combinatorial_lowerbound}}
Consider a genie-aided scenario where the set of flipped entries is revealed as side information, which is equivalent to saying that we are in the noiseless setting with $p = 0$. Then the true partition corresponding to the true user cluster structure has zero disagreement. Suppose users $1, 3, \ldots, 2K-1$ are in true cluster $1$ and users $2,4, \ldots, 2K$ are in true cluster $2$. We construct a  new partition different from the true partition by swapping user $1$ and user $2$. In particular, under the new partition, user $1$ forms a new cluster $\widehat{C}_2$ with users $2i, i=2,\ldots, K$, user $2$ forms a new cluster $\widehat{C}_1$ with users $2i-1, i=2, \ldots,K$. It suffices to show that for $k=1,2$, any two users in $\widehat{C}_k$ has zero disagreement with probability at least $3/4$, in which case the new partition has zero agreement and Algorithm \ref{alg:combinatorialmethod} cannot distinguish between the true partition and the new one.

For $k=1,2$, we lower bound the probability that any two users in $\widehat{C}_k$ has zero disagreement.
   \begin{align*}
    &\mathbb{P}(\text{Any two users in $\widehat{C}_k$ has zero disagreement}) \\
     = &1-\mathbb{P}(\text{total number of disagreements in $\widehat{C}_k$ $\geq 1$})\\
     \geq & 1-\mathbb{E}[\text{total number of disagreements in $\widehat{C}_k$}]\\
     \geq & 1-\frac{1}{2}nK(1-\epsilon)^2\geq 7/8.
   \end{align*}
By union bound, the probability that for $k=1,2$, any two users in $\widehat{C}_k$ has zero disagreement is at least $3/4$.

\subsection{Proof of Theorem \ref{result:noiseless_clustering}}
Consider a compatibility graph with $n$ vertices representing users. Two vertices $i,i'$ are connected if users $i,i'$ have zero disagreement, i.e., $D_{ii'}=0$. In the noiseless setting, each user cluster forms a clique of size $K$ in the compatibility graph. We call a clique of size $K$ in the compatibility graph a bad clique if it is formed by users from more than one cluster. Then to prove the theorem, it suffices to show that there is no bad clique a.a.s. Since the probability that bad cliques exist increases in $\epsilon$, without loss of generality, we assume $K(1-\epsilon)<1$.

Recall that $B_{kl}$ is $+1$ or $-1$ with equal probability. Define $S_k = \{l: B_{kl} = +1\}$ for $k=1, \ldots, r$. As $r\to \infty$, by Chernoff bound, we get that a.a.s., for any $k_1 \ne k_2$
\begin{align}
  |S_{k_1}\Delta S_{k_2} | \triangleq | \{ l: B_{k_1l} \ne B_{k_2l}\} | \geq \frac{r}{4}. \label{Eq:Concentration}
\end{align}
Assume this condition holds throughout the proof.

Fix a set of $K$ users that consists of users from $t$ different clusters. Without loss of generality, assume these users are from cluster $1, \dots, t$. Let $n_k$ denote the number of users from the cluster $k$ and define $n_{\max} = \max_k n_k$. By definition, $2\leq t\leq t_{\max}\triangleq \min \{r, K\}$, $n_{\max}<K$ and $\sum_{k=1}^t n_k =K$. For any movie $j$ in cluster $l$, among the $K$ ratings given by these users, there are $\sum_{k = 1}^t{n_k\1{l\in S_k}}$ ratings being $+1$ and $\sum_{k = 1}^t{n_k\1{l\not\in S_k}}$ ratings being $-1$. Let $E_j$ denote the event that the observed ratings for movie $j$ by these $K$ users are the same. Then,
\begin{align*}
   \mathbb{P}[E_j] =&1-\left(1-\epsilon^{\sum_{k = 1}^t{n_k\1{l\in S_k}}}\right)\left(1-\epsilon^{\sum_{k = 1}^t{n_k\1{l\not\in S_k}}}\right)\\
   \le& \exp \left( - (1-\epsilon^{\sum_{k = 1}^t{n_k \1 {l \in S_k}}} ) ( 1-\epsilon^{\sum_{k = 1}^t{n_k \1{ l\not\in S_k}}})  \right)  \\
   \le & \exp \Big( -\frac{1}{4} (1-\epsilon)^2  \sum_{k = 1}^t{n_k \1 {l \in S_k}}   \sum_{k = 1}^t{n_k \1 {l \not\in S_k}} \Big).
\end{align*}
Let $p_{n_1\dots n_t}$ be the probability that $K$ users, out of which $n_k$ are from cluster $k$, form a bad clique. Because $\{E_j\}$ are independent and there are $K$ movies in each movie cluster,
\begin{align}
    p_{n_1\dots n_t}
    \le& \exp \Big( -\frac{1}{4} K (1-\epsilon)^2 \sum_{l= 1}^r ( \sum_{k = 1}^t{n_i\1{l\in S_k}}\sum_{k = 1}^t{n_k \1{l\not\in S_k}} ) \Big) \nonumber \\
    =& \exp \Big( -\frac{1}{4} K (1-\epsilon)^2  \sum_{1 \le k_1<k_2 \le t} n_{k_1} n_{k_2} |S_{k_1}\Delta S_{k_2} | \Big) \nonumber \\
    \le &\exp \Big( -C_1 n (1-\epsilon)^2  \sum_{k = 1}^t n_{k}(K-n_k)  \Big) \label{EqBoundProbClique}
   \end{align}
for some constant $C_1$. For a large enough constant $C$ in the assumption regarding $m$ in the statement of the theorem and using the fact that $m=n^2(1-\epsilon)$, we have
  \begin{align}
     K \exp(-C_1n(1-\epsilon)^2(K-n_k))\leq & n^{-3}, \quad n_k\leq \frac{K}{2} , \label{EqBadCliqueCond1} \\
     K \exp(-C_1n(1-\epsilon)^2n_k)\leq & n^{-3}, \quad n_k> \frac{K}{2}.\label{EqBadCliqueCond2}
  \end{align}

  Below we show that the probability of bad cliques existing goes to zero. By the Markov inequality and linearity of expectation,
  \begin{align*}
    & \ \mathbb{P}[\text{Number of bad cliques $\geq 1$}]\nonumber \\
    \leq &\ \mathbb{E}[\text{Number of bad cliques}]\\
    \leq & \sum_{t = 2}^{t_{\max}}\binom{r}{t}\sum_{n_1+\cdots+n_t = K}\binom{K}{n_1}\cdots \binom{K}{n_t}p_{n_1\dots n_t} \nonumber \\
    \stackrel{(s)}{\leq} & \sum_{t = 2}^{t_{\max}} r^tK^t n^{-3K}+\sum_{t = 2}^{t_{\max}} \1{t\leq K-n_{\max}+1} r^t K^t n^{-6(K-n_{\max})} \nonumber \\
    = &o(1),
   \end{align*}
  where the first term in last inequality corresponds to the case of $n_{\max} \le K/2$ and the second term corresponds to the case of $n_{\max}>K/2$. They follows from \eqref{EqBoundProbClique}, \eqref{EqBadCliqueCond1}, \eqref{EqBadCliqueCond2} and the fact that $\binom{K}{n_k}\leq \min\{K^{n_k}, K^{K-n_k}\}$.

\subsection{Proof of Theorem \ref{result:convex_relaxation}}
We first introduce some notations. Let $u_{C, k}$ be the normalized characteristic vector of user cluster $k$, i.e., $u_{C, k}(i) = 1/\sqrt{K}$ if user $i$ is in cluster $k$ and $u_{C, k}(i) =0$ otherwise. Thus, $||u_{C, k}||_2 = 1$. Let $U_C = [u_{C, 1}, \dots, u_{C, r}]$. Similarly, let $v_{C, l}$ be the normalized characteristic vector of movie cluster $l$ and $V_C = [v_{C, 1}, \dots, v_{C, r}]$. It is not hard to see that the rating matrix $R$ can be written as $R = KU_CBV_C^\top$. Denote the SVD of the block rating matrix $B$ by $B = U_B\Sigma_B V_B^\top$, then the SVD of $R$ is simply $R = UK\Sigma_BV^\top$, where $U = U_CU_B$ and $V = V_CV_B$. When $r\to \infty$, $B$ has full rank almost surely \cite{BVW10}. We will assume $B$ is full rank in the following proofs, which implies that $U_BU_B^\top = I$ and $V_BV_B^\top = I$. Note that $UU^\top = U_CU_C^\top, VV^\top = V_CV_C^\top$ and $UV^\top = U_CU_BV_B^\top V_C^\top$.

We now briefly recall the subgradient of the nuclear norm \cite{Candes12}. Define $T$ to be the subspace spanned by all matrices of the form $UA^\top$ or $AV^\top$ for any $A\in \mathbb{R}^{n\times r}$. The orthogonal projection of any matrix $M\in \mathbb{R}^{n\times n}$ onto the space $T$ is given by $\mathcal{P}_T(M) = UU^\top M+MVV^\top-UU^\top MVV^\top $. The projection of $M$ onto the complement space $T^\perp$ is $\mathcal{P}_{T^\perp}(M) = M - \mathcal{P}_T(M)$. Then $M\in \mathbb{R}^{n\times n}$ is a subgradient of $||X||_*$ at $X = R$ if and only if $\mathcal{P}_T(M) = UV^\top $ and $||\mathcal{P}_{T^\perp}(M)||\leq 1$.

\begin{proof}[of Lemma~\ref{result:incoherence_loose}]
  Assume user $i$ is from user cluster $k$ and movie $j$ is in movie cluster $l$, then
  \begin{align*}
    |(UV^\top)_{ij}| = |(U_BV_B^\top)_{kl}| / K \leq 1/K = r/n,
  \end{align*}
  where the inequality follows from the Cauchy-Schwartz inequality. By definition $\mu\leq \sqrt{r}$.
\end{proof}

Next we establish the concentration property of $\widehat{R}$. By definition the conditional expectation of $\widehat{R}$ is given by
$
\mathbb{E}[\widehat{R} | R] = (1-\epsilon) (1-2p)  R:=\bar{R}.
$
Furthermore, the variance is given by
$
\text{Var} [ \widehat{R}_{ij} |R] =   (1-\epsilon)- (1-\epsilon)^2 (1-2p)^2 :=\sigma^2.
$

The following corollary applies Theorem 1.4 in \cite{Vu07} to bound the spectral norm $\| \widehat{R}- \bar{R}\|$.
\begin{corollary} \label{result:concentration}
If $\sigma^2 \ge C^\prime \log^4 n /n$ for a constant $C^\prime$, then conditioned on $R$,
\begin{align}
\| \widehat{R}-\bar{R} \| \le 3 \sigma \sqrt{n}\quad a.a.s.
\end{align}
\end{corollary}
\begin{proof}
We adopt the trick called dilations \cite{Tropp12}. In particular, define $A$ as
\begin{align}
A=\left[ \begin{array}{cc}
\mathbf{0} & \widehat{R}- \mathbb{E}[\widehat{R} | R]\\
\widehat{R}^\top-\mathbb{E}[\widehat{R}^\top | R] & \mathbf{0}
\end{array}
\right].
\end{align}
Observe that $\|A\| = \| \widehat{R}- \mathbb{E}[\widehat{R} | R] \|$, so it is sufficient to prove the theorem for $\|A\|$. Conditioned on $R$, $A$ is a random symmetric $2n \times 2n$ matrix with each entry bounded by $1$, and $a_{ij}$ $(1 \le i <j \le 2n)$ are independent random variables with mean $0$ and variance \emph{at most} $\sigma^2$.
By Theorem 1.4 in \cite{Vu07} , if $\sigma \ge C^\prime n^{-1/2}  \log^2 n$, then conditioned on $R$ a.a.s.
\begin{align}
\| \widehat{R}- \mathbb{E}[\widehat{R} | R] \| = \|A\| &\le 2 \sigma \sqrt{2n} + C (2 \sigma)^{1/2} (2n)^{1/4} \log (2n) \nonumber  \\
& \le 3 \sigma \sqrt{n},
\end{align}
where the last inequality holds when $n$ is sufficiently large.
\end{proof}

\begin{proof}[of Theorem~\ref{result:convex_relaxation}] 
  For any feasible $Y$ that $Y\ne R$, we have to show that
  $
    \Delta(Y) = \langle\widehat{R}, R\rangle-\lambda||R||_*-(\langle\widehat{R}, Y\rangle-\lambda||Y||_*)>0.
  $
  Rewrite $\Delta(Y)$ as
  \begin{align}
    \Delta(Y) = & \langle\bar{R}, R-Y\rangle+\langle\widehat{R}-\bar{R}, R-Y\rangle \nonumber \\
    & +\lambda(||Y||_*-||R||_*). \label{EqDeltaY}
  \end{align}
  The first term in \eqref{EqDeltaY} can be written as
  \begin{align*}
    \langle\bar{R}, R-Y\rangle &= (1-\epsilon)(1-2p)\langle R, R-Y\rangle \nonumber  \\
    &= (1-\epsilon)(1-2p)||R-Y||_1,
  \end{align*}
  where the second equality follows from the fact that $Y_{ij}\in [-1, 1]$ and $R_{ij} = \mathrm{sgn}(R_{ij})$.
  Define the normalized noise matrix $W = (\widehat{R}-\bar{R})/\lambda$. Note that $||W||_\infty\leq 1/\lambda$ and $\mathrm{Var}(W_{ij})\leq 1/9n$. The second term in \eqref{EqDeltaY} becomes $\langle\widehat{R}-\bar{R}, R-Y\rangle = \lambda\langle W, R-Y\rangle$.
  By Corollary~\ref{result:concentration}, $||W||\leq 1$ almost surely. Thus $UV^\top+\mathcal{P}_{T^\perp}(W)$ is a subgradient of $||X||_*$ at $X = R$. Hence, for the third term in \eqref{EqDeltaY},
  $\lambda(||Y||_*-||R||_*)\geq \lambda\langle UV^\top+\mathcal{P}_{T^\perp}(W), Y-R\rangle$. Therefore,
  \begin{align}
    &\Delta(Y) \nonumber \\
    &\geq (1-\epsilon)(1-2p)||R-Y||_1+\lambda\langle UV^\top -\mathcal{P}_{T}(W), Y-R\rangle \nonumber\\
    &\geq  [(1-\epsilon)(1-2p)-\lambda(||UV^\top||_\infty+||\mathcal{P}_{T}(W)||_\infty)]||R-Y||_1 \nonumber\\
    &\geq  [(1-\epsilon)(1-2p)-\lambda(\mu\sqrt{r}/n+||\mathcal{P}_{T}(W)||_\infty)]||R-Y||_1, \label{EqDeltaYLowerBound}
  \end{align}
  where the last inequality follows from definition of the incoherence parameter $\mu$. Below we bound the term $||\mathcal{P}_{T}(W)||_\infty$. From the definition of $\mathcal{P}_T$ and the fact that $U_BU_B^\top = I$ and $V_BV_B^\top = I$,
  \begin{align*}
    ||\mathcal{P}_T(W)||_\infty 
    \leq & ||U_CU_C^\top W||_\infty+||WV_CV_C^\top||_\infty \\
    & +||U_CU_C^\top WV_CV_C^\top||_\infty.
  \end{align*}
  We first bound $||U_CU_C^\top W||_\infty$. To bound the term $(U_CU_C^\top W)_{ij}$, assume user $i$ belongs to user cluster $k$ and let $\mathcal{C}_k$ be the set of users in user cluster $k$. Recall that $u_{C, k}$ is the normalized characteristic vector of user cluster $k$. Then
  \begin{align*}
    (U_CU_C^\top W)_{ij} = (u_{C, k}u_{C, k}^\top W)_{ij} = (1/K) \sum_{i'\in \mathcal{C}_k}W_{i'j},
  \end{align*}
  which is the average of $K$ independent random variables. By Bernstein's inequality (stated in the supplementary material), with probability at least $1-n^{-3}$,
  \begin{align*}
    \left|\sum_{i'\in \mathcal{C}_k}W_{i'j}\right|\leq \sqrt{\frac{2}{3r}\log n}+\frac{2\log n}{\lambda}.
  \end{align*}
  Then $||U_CU_C^\top W||_\infty\leq \frac{1}{K}\left(\sqrt{\frac{2}{3r}\log n}+\frac{2\log n}{\lambda}\right)$ with probability at least $1-n^{-1}$. Similarly we bound $||WV_CV_C^\top ||_\infty$ and $||U_CU_C^\top WV_CV_C^\top ||_\infty$. Therefore, with probability at least $1-3n^{-1}$,
  \begin{align}
    ||\mathcal{P}_T(W)||_\infty\leq \frac{C_1}{K}\left(\sqrt{\frac{\log n}{r}}+\frac{\log n}{\lambda}\right)
    \leq \frac{C_2}{K}\sqrt{\frac{\log n}{r}}, \label{EqPTW}
  \end{align}
  for some constants $C_1$ and $C_2$, where the second inequality follows from assumption~\eqref{EqCondition}. Substituting \eqref{EqPTW} into \eqref{EqDeltaYLowerBound} and by assumption~\eqref{EqCondition} again, we conclude that $\Delta(Y)>0$ a.a.s.
\end{proof}

\subsection{Proof of Theorem \ref{ThmSpectralMethod}}
The proof is divided into three parts. Recall that $x_i$ denotes the $i$-th row of $Pr(\widehat{R}^{(1)})$. We first show that, for most users, ${x}_i$ is close to the expected value conditioned on $R$. Then we show that the clusters output by Step $2$ are close to the true clusters. Finally, we show that Step $3$ exactly recovers the block rating matrix $B$ and Step $4$ exactly recovers clusters.

Define $\bar{R}^{(1)} = \E{\widehat{R}^{(1)}|R} = \frac{1}{2} (1-\epsilon) (1-2p) R$ and let $\bar{x}_i$ be the $i$-th row of $\bar{R}^{(1)}$. We call user $i$ a {\em good} user if $\|x_i-\bar{x}_i\|_2 \le  \tau/2$ where the threshold $\tau=12 (1-\epsilon)^{1/2} \log n$; otherwise it is called a {\em bad} user. Let $\mathcal{I}$ denote the set of all good users and $\mathcal{I}^c$ denote the set of all bad users. Define good movies in the same way, and let $\mathcal{J}$ denote the set of all good movies and $\mathcal{J}^c$ denote the set of all bad movies.  The following lemma shows that the number of bad users (movies) are bounded by $K \log^{-2} n$.

\begin{lemma}\label{LemmaSpectralBound}
If $\sigma^2 \ge C' \log^4 n/ n$ for a constant $C'$, then a.a.s., $|\mathcal{I}^c| \le K \log^{-2}n$ and $|\mathcal{J}^c| \le K \log^{-2}n$.
\end{lemma}
\begin{proof}
Let $(\sigma^{(1)})^2=\frac{1}{2}(1-\epsilon)$. By Corollary \ref{result:concentration}, $\|\widehat{R}^{(1)}-\bar{R}^{(1)}\| \le 3 \sigma^{(1)} \sqrt{n}$. Note that
\begin{align}
\|P_r(\widehat{R}^{(1)} )-\bar{R}\| & \le \|P_r(\widehat{R}^{(1)} )-\widehat{R}^{(1)}\|+\|\widehat{R}^{(1)}-\bar{R}\| \nonumber \\
& \le 2 \|\widehat{R}^{(1)}-\bar{R}\|, \nonumber
\end{align}
where the second inequality follows from the definition of $P_r(\widehat{R}^{(1)})$ and the fact that $\bar{R}$ has rank $r$. Since both $P_r(\widehat{R}^{(1)} )$ and $\bar{R}$ have rank $r$, the matrix $P_r(\widehat{R}^{(1)} )-\bar{R}$ has rank at most $2r$, which implies that
\begin{align*}
\|P_r(\widehat{R}^{(1)} )-\bar{R}\|^2_F \le 8r \|\widehat{R}^{(1)}-\bar{R}\|^2\leq 72 (\sigma^{(1)})^2 n r.
\end{align*}
As $\sum_{i=1}^n \|x_i - \bar{x}_i \|^2_2 = \|P_r(\widehat{R}^{(1)} )-\bar{R}\|^2_F$, we conclude that there are at most $K \log^{-2} n$  users with $$\|x_i-\bar{x}_i\|_2 > 6\sqrt{2} \sigma^{(1)} r \log n = \tau/2.$$ Similarly we can prove the result for movies.
\end{proof}

The following proposition upper bounds the set difference between the estimated clusters and the true clusters by $K\log^{-2} n$. Let $C^*_1, \ldots, C^*_r$ be the true user clusters and $\Delta$ denote the set difference.

\begin{proposition} \label{PropSpectralRecovery}
Assume the assumption of Theorem \ref{ThmSpectralMethod} holds. Step 2 of Algorithm \ref{alg:spectralmethod} outputs $\{\widehat{C}_{k} \}_{k=1}^r$ and $\{\widehat{D}_{l}\}_{1=1}^r$ such that, up to a permutation of cluster indices, a.a.s., $\widehat{C}_k \Delta C^*_k \subset \mathcal{I}^c$ and $\widehat{D}_l \Delta D^*_l \subset \mathcal{J}^c$ for all $k, l$. It follows that for all $k,l$,
\begin{align}
| \widehat{C}_{k} \Delta C^*_k | \le \frac{K}{\log^2 n}, \quad  | \widehat{D}_{l} \Delta D^*_l | \le \frac{K}{\log^2 n}. \label{EqSpectralApprox}
\end{align}
\end{proposition}
\begin{proof}
It suffices to prove the conclusion for the user clusters. Consider two good users $i,i^\prime \in \mathcal{I}$. If they are from the same cluster, we have $\bar{x}_i = \bar{x}_{i'}$ and
\begin{align}
\| x_i -x_{i^\prime} \| \le \| x_i - \bar{x}_{i} \|+ \| x_{i^\prime} - \bar{x}_{i^\prime} \| \le \tau, \label{Eq:SpectralSame}
\end{align}
where the last inequality follows from Lemma~\ref{LemmaSpectralBound}. If they are from different clusters, by~\eqref{Eq:Concentration}, we have a.a.s.
\begin{align*}
\| \bar{x}_i - \bar{x}_{i^\prime} \|^2_2 = &\frac{1}{4}  (1- \epsilon)^2 (1-2p)^2 ||R_i - R_{i'}||_2^2\\
\ge &\frac{1}{4}  (1- \epsilon)^2 (1-2p)^2 n,
\end{align*}
where $R_i$ denotes the $i$-th row of $R$. Thus,
\begin{align}
\| x_i -x_{i^\prime} \| & \ge \| \bar{x}_i - \bar{x}_{i^\prime} \|- \| x_i - \bar{x}_{i} \|-\| x_{i^\prime} - \bar{x}_{i^\prime} \| \nonumber \\
&\ge \frac{1}{2}(1-\epsilon) (1-2 p) \sqrt{n}    -  \tau > \tau, \label{Eq:SpectralDiff}
\end{align}
where the last inequality follows from the assumption~\eqref{EqSpectralCondition}. Therefore, in the clustering procedure of Step $2$, if we choose a good initial user at some iteration, the corresponding estimated cluster will contain all the good users from the same cluster as the initial user and no good user from other clusters. It is not hard to see that the probability of the event that we choose a good initial user in every iteration is lower bounded by
\begin{align*}
& \left(1-\frac{1}{r\log^2 n}\right)\left(1-\frac{1}{(r-1)\log^2 n}\right)\dots\left(1-\frac{1}{\log^2 n}\right)\\
& \geq  1-\frac{1}{ \log^2 n} \left( \frac{1}{r} + \frac{1}{r-1} + \cdots + 1 \right) \\
& \geq  1-\frac{\log r}{\log^2 n} \geq 1 - \frac{1}{\log n}.
\end{align*}
Assume the above event holds. Under proper permutation, the initial good user in the $k$-th iteration is from cluster $C_k^*$ for all $k$. By the above argument, the set difference $\widehat{C}_k \Delta C^*_k \subset \mathcal{I}^c$. By Lemma \ref{LemmaSpectralBound}, ~\eqref{EqSpectralApprox} follows.
\end{proof}

\begin{proof}[of Theorem \ref{ThmSpectralMethod}]
We first show that Step $3$ of Algorithm \ref{alg:spectralmethod} exactly recovers the block rating matrix $B$. Let $V_{kl}$ denote the total
vote that the true user cluster $k$ gives to the true movie cluster $l$, i.e., $$V_{kl}= \sum_{i \in C^*_k} \sum_{j \in D^*_l} \widehat{R}^{(2)}_{ij}.$$ Then by definition of $\widehat{V}_{kl}$,
\begin{align}
| \widehat{V}_{kl} - V_{kl} |  \le & \sum_{i \in C_k^* \Delta \widehat{C}_k } \sum_{j \in D_l^* \cup \widehat{D}_l} \1{(i,j) \in \Omega_2} \nonumber \\
&+ \sum_{i \in C_k^* \cup \widehat{C}_k } \sum_{j \in D_l^* \Delta \widehat{D}_l} \1{(i,j) \in \Omega_2}. \label{Eq:VoteUpperBound}
\end{align}

Without loss of generality, assume $B_{kl} = 1$. By Bernstein inequality and assumption (\ref{EqSpectralCondition}), $ V_{kl} \ge \frac{1}{4}(1- \epsilon) (1-2 p)K^2$ a.a.s. On the other hand, as $\Omega_2$ and $\widehat{R}^{(1)}$ are independent, $\Omega_2$ is independent from $ \{\widehat{C}_k \}$ and $\{ \widehat{D}_l \}$. It follows from \eqref{EqSpectralApprox} and the Chernoff bound that each term on the right hand side of~\eqref{Eq:VoteUpperBound} is upper bounded by $(1-\epsilon)K^2 \log^{-2}n $ a.a.s. Hence, when assumption~\eqref{EqSpectralCondition} holds for some large enough constant $C$, we have $\widehat{V}_{kl}>0$ thus $\widehat{B}_{kl}=B_{kl}$.

Next we prove that Step $4$ clusters the users and movies correctly. Without loss of generality, we only prove the correctness for users. Suppose user $i$ is from cluster $k$. Recall that $R_i$ denotes the $i$-th row of $R$. When $\widehat{B} = B$, we have $\mu_{kj} = R_{ij}$ for $j\in \mathcal{J}$ by definition  and Proposition \ref{PropSpectralRecovery}. Then
\begin{align}
\langle \widehat{R}^{(2)}_i, \mu_k \rangle = &\langle \widehat{R}^{(2)}_i, R_i \rangle + \langle \widehat{R}^{(2)}_i, \mu_k-R_i \rangle \nonumber \\
\geq & \langle \widehat{R}^{(2)}_i, R_i \rangle - 2 \sum_{j \in  \mathcal{J}^c } | \widehat{R}^{(2)}_{ij}|. \label{eq:center1}
\end{align}
Similarly, for some user $i'$ from cluster $k'\ne k$,
\begin{align}
\langle \widehat{R}^{(2)}_i, \mu_{k'} \rangle = &\langle \widehat{R}^{(2)}_i, R_{i'} \rangle + \langle \widehat{R}^{(2)}_i, \mu_{k'}-R_{i'} \rangle \nonumber \\
\leq & \langle \widehat{R}^{(2)}_i, R_{i'} \rangle+2 \sum_{j \in \mathcal{J}^c } | \widehat{R}^{(2)}_{ij}| \label{eq:center2}
\end{align}
For ease of notation, let $t:=\frac{1}{2}(1-\epsilon) (1-2p)n$ and $(\sigma^{(2)})^2 = \frac{1}{2}(1-\epsilon)$.  By~\eqref{Eq:Concentration}, $\langle R_i, R_{i^\prime} \rangle \le n/2$ for all $i \neq i'$. Then conditioned on $R$, we have $\mathbb{E} [\langle \widehat{R}^{(2)}_i, R_i \rangle ] =  t $ and $\text{Var}[\langle \widehat{R}^{(2)}_i, R_i \rangle] \le n (\sigma^{(2)})^2$, and
\begin{align*}
\mathbb{E} [ \langle \widehat{R}^{(2)}_i, R_{i^\prime} \rangle  ]  =\frac{1}{2} (1- \epsilon) (1-2 p) \langle R_i, R_i^\prime \rangle \le t/2
\end{align*}
and $\text{Var}[\langle \widehat{R}^{(2)}_i, R_{i'} \rangle ] \le n (\sigma^{(2)})^2 $. Now by the Bernstein inequality and assumption (\ref{EqSpectralCondition}), we have that conditioned on $R$, a.a.s. $\langle  \widehat{R}^{(2)}_i, R_i \rangle > 7t/8 $ and  $  \langle \widehat{R}^{(2)}_i, R_{i^\prime} \rangle <5t/8$ for all $i \neq i'$.

On the other hand, because $\mathcal{J}$ and $\Omega_2$ are independent, by the Chernoff bound, a.a.s. $\sum_{j \in \mathcal{J}^c } |\widehat{R}^{(2)}_{i j } |$ is upper bounded by $(1-\epsilon) K \log^{-2} n<t/16$ for all $i$, when assumption~\eqref{EqSpectralCondition} holds for some large enough constant $C$.

Therefore, from \eqref{eq:center1} and \eqref{eq:center2}, $\langle \widehat{R}^{(2)}_i, \mu_k \rangle> \langle \widehat{R}^{(2)}_i, \mu_{k^\prime} \rangle $ for all $k' \neq k$.
\end{proof}

\section{Numerical Experiments} \label{Sec: Sim}
In this section, we illustrate the performance of the convex method and the spectral method using synthetic data.

\subsection{Convex method}
The convex program~\eqref{LowRank} can be formulated as a semidefinite program (SDP) and solved using a general purpose SDP solver. However this method does not scale well for our problem when the matrix dimension $n$ is large. Instead we propose a first-order algorithm motivated by the Singular Value Thresholding algorithm proposed in \cite{Cai10}. Consider the following convex program which introduces an additional term $\frac{\tau}{2}||Y||_F^2$ for $\tau>0$,
\begin{align}
  \min \quad & \frac{\tau}{2}||Y||_F^2-\langle \widehat{R}, Y\rangle +\lambda||Y||_* \nonumber \\
  \text{s.t.} \quad & Y_{ij}\in [-1, 1]. \label{equ:convex_prog_SVT}
\end{align}
When $\tau$ is small, the solution of \eqref{equ:convex_prog_SVT} is close to that of~\eqref{LowRank}. We solve the convex program \eqref{equ:convex_prog_SVT} using the dual gradient descent method given in Algorithm \ref{alg:convex_SVT}. Let matrix $E$ denote the matrix of all ones. Let matrices $S \ge 0 $ \footnote{For two matrices $X$ and  $X'$,  $X \ge X'$ means that $X_{ij} \ge X'_{ij}$ for all entries $(i,j)$. } and $T \ge 0$  denote the Lagrangian multipliers for the constraints $Y\leq  E$ and $Y\geq -E$, respectively. The Lagrangian function is given by
\begin{align*}
L(Y, S, T) = &  \frac{\tau}{2}||Y||_F^2-\langle \widehat{R}, Y\rangle +\lambda||Y||_*  \\
 & + \langle S, Y- E \rangle - \langle T, Y+E \rangle,
\end{align*}
and the dual function is given by  $\min_{Y} L(Y, S,T)$. The gradients of the dual function with respect to $S$ and $T$ are $Y- E$ and  $-Y-E$, respectively.

The minimizer of the Lagrangian function $L(Y,S, T)$ for fixed $S$ and $T$ can be explicitly written in terms of the soft-thresholding operator $D$ defined as follows. For any $\gamma\geq 0$ and for any matrix $X$ with SVD $X = U\Sigma V^\top$ where $\Sigma = \mathrm{diag}(\{\sigma_i\})$,
define
$$D_\gamma(X) = U\mathrm{diag}(\{\max(\sigma_i-\gamma, 0)\})V^\top.$$
Intuitively, the soft-thresholding operator $D$ shrinks the singular values of $X$ towards zero.  Applying Theorem 2.1 in \cite{Cai10}, we get
\begin{align*}
D_{\frac{\lambda}{\tau}}\left(\frac{\widehat{R}-S+T}{\tau}\right) = \arg \min_{Y} L (Y, S, T).
\end{align*}
Thus, the first update equation in~\eqref{eq: Iteration} finds the minimizer of the Lagrangian function at the current estimate of the dual variables.  The next two update equations in~\eqref{eq: Iteration} move the current estimate of the dual variables in the direction of the corresponding gradients of the dual function, and then project the new estimates to the set of matrices with nonnegative entries. The parameter $\delta>0$ is the step size.

\begin{algorithm}
\caption{Dual Gradient Descent Algorithm for \eqref{equ:convex_prog_SVT}}\label{alg:convex_SVT}
\begin{algorithmic}
\STATE \textbf{Input}: $\widehat{R}$
\STATE \textbf{Initialization}: $Y_0 = 0; S_0 = 0; T_0 = 0; \lambda>0; \tau, \delta>0. $
\WHILE {not converge}
\STATE
\begin{align}
  Y_{k+1} &=D_{\frac{\lambda}{\tau}}\left(\frac{\widehat{R}-S_k+T_k}{\tau}\right) \nonumber  \\
  S_{k+1} &= \max\{S_k+\delta(Y_k-E), 0\}  \nonumber \\
  T_{k+1}  &= \max\{T_k-\delta(Y_k+E), 0\} \label{eq: Iteration}
\end{align}
\ENDWHILE
\end{algorithmic}
\end{algorithm}

We simulate Algorithm~\ref{alg:convex_SVT} on the
synthetic data.  Assume $K$ and $\epsilon$ take the form given by
\begin{align}
K = n^\beta, \quad \epsilon = 1-n^{-\alpha} \label{Eq:form}.
\end{align}
Theorem~\ref{result:convex_relaxation} shows that the convex program~\eqref{EqCondition} recovers the rating matrix exactly when $\alpha<\beta$, assuming Conjecture~\ref{result:conjecture} holds.

We generate the observed data matrix with $n = 2048$, $p = 0$ and various choices of $\beta, \alpha\in (0, 1)$, and apply Algorithm~\ref{alg:convex_SVT}. The solution $\widehat{Y}$  is evaluated by the fraction of entries with correct signs, i.e., $\frac{1}{n^2}|\{(i, j): \mathrm{sign}(\widehat{Y}_{ij}) = R_{ij}\}|$. The result is plotted in grey scale in Figure~\ref{fig:convex_simulation}. In particular, the white area represents exact recovery and the black area represents around $50\%$ recovery, which is equivalent to random guess. The red line represents $\alpha=\beta$, which shows the performance guarantee given by Theorem \ref{result:convex_relaxation}. As we can see, the simulation results roughly match the theoretical performance guarantee.
\begin{figure}
\centering
\post{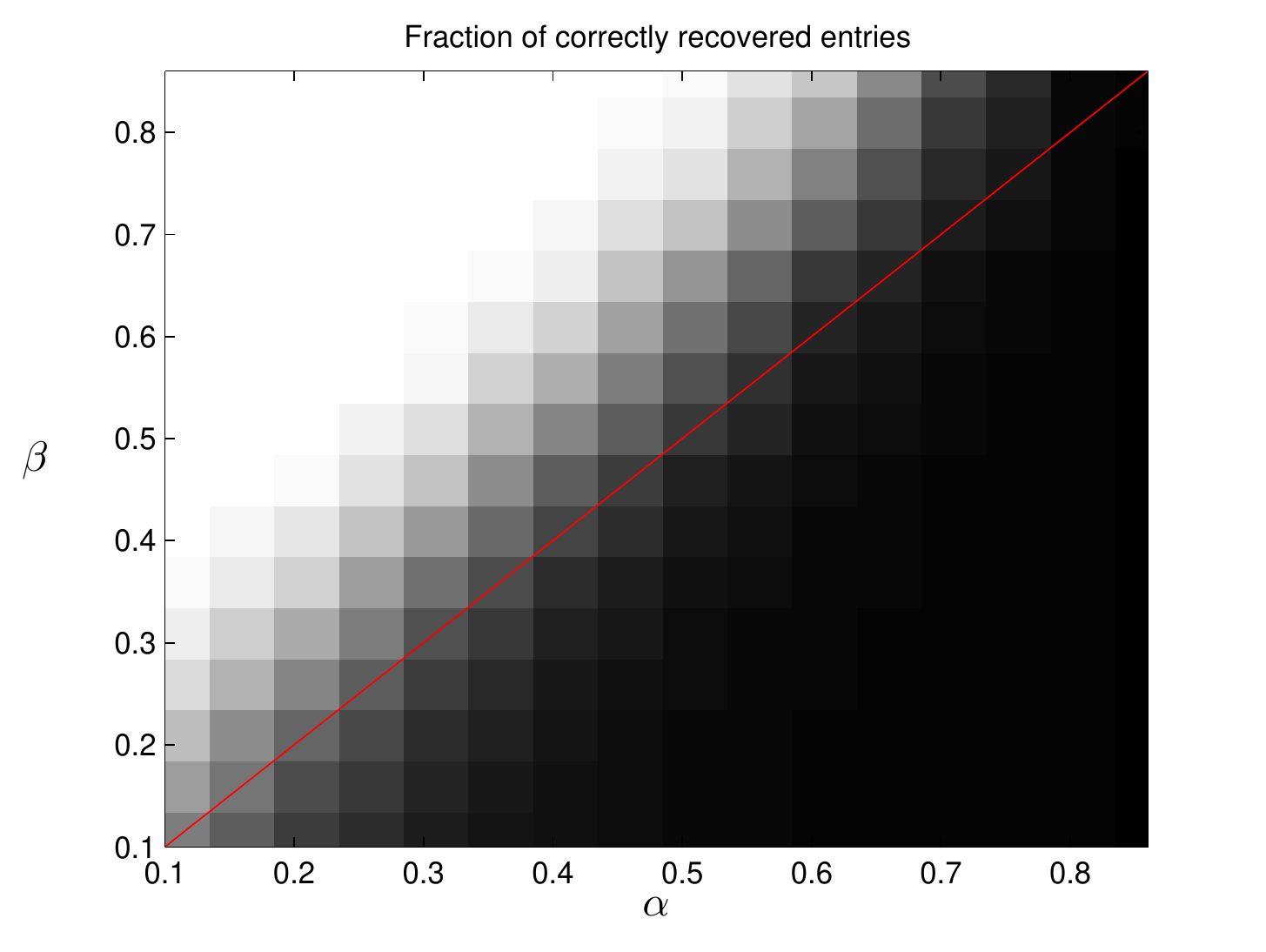}{3.2in}
\centering
\caption{Simulation result of the convex method given in Algorithm ~\ref{alg:convex_SVT} with $n=2048$ and $p=0$. The $x$-axis corresponds to erasure probability $\epsilon = 1-n^{-\alpha}$ and $y$-axis
corresponds to cluster size $K= n^\beta$. The grey scale of each area represents the fraction of entries with correct signs, with white representing exact recovery and black representing around $50 \%$ recovery. The red line shows the performance of the convex method predicted by Theorem \ref{result:convex_relaxation}.}
\label{fig:convex_simulation}
\end{figure}

\subsection{Spectral Method}
We simulate the spectral method given in Algorithm \ref{alg:spectralmethod} on synthetic data. Assume $K$ and $\epsilon$ take the form of~\eqref{Eq:form}. Theorem \ref{ThmSpectralMethod}  shows that the spectral method exactly recovers the clusters when $\alpha < \frac{1}{2} (\beta+1)$.

We generate the observed data matrix according to our model with $n = 1000$ and $p = 0.05$, and  various choices of $\beta, \alpha\in (0, 1)$. We apply Algorithm \ref{alg:spectralmethod} with slight modifications. Firstly, we do not split the observation as in Step $1$ but use all the observations for the later steps, i.e., $\Omega_1 = \Omega_2 = \Omega$.  Secondly, in Step $2$ we use the more robust $k$-means algorithm to cluster users and movies instead of the thresholding based clustering algorithm.

To calculate the number of mis-clustered users (movies), we need to consider all possible permutations of the cluster indices, which is computationally expensive for large $r$. Thus,
the clustering error is instead measured by the fraction of misclassified pairs of users and movies.  In particular, we say a pair of users (movies) misclassified if they are either from the same true cluster but assigned to two different clusters or from two different true clusters but assigned to the same cluster. We say the algorithm succeeds if the clustering error is less than $5\%$.

For each $\beta$, we run the algorithm for several values of $\alpha$ and record the largest $\alpha$ for which the algorithm succeeds. The result is depicted in Fig~\ref{fig:spectral_simulation}. The solid blue line represents $\alpha=\frac{1}{2} (\beta+1)$, which shows the performance guarantee of the spectral method given by Theorem \ref{ThmSpectralMethod}.
The dotted red line represents $\alpha=\beta$, which shows the  performance guarantee of the convex method given by Theorem \ref{result:convex_relaxation}.
We can see that our simulation results are slightly better than the theoretical performance guarantee.
\begin{figure}
\centering
\post{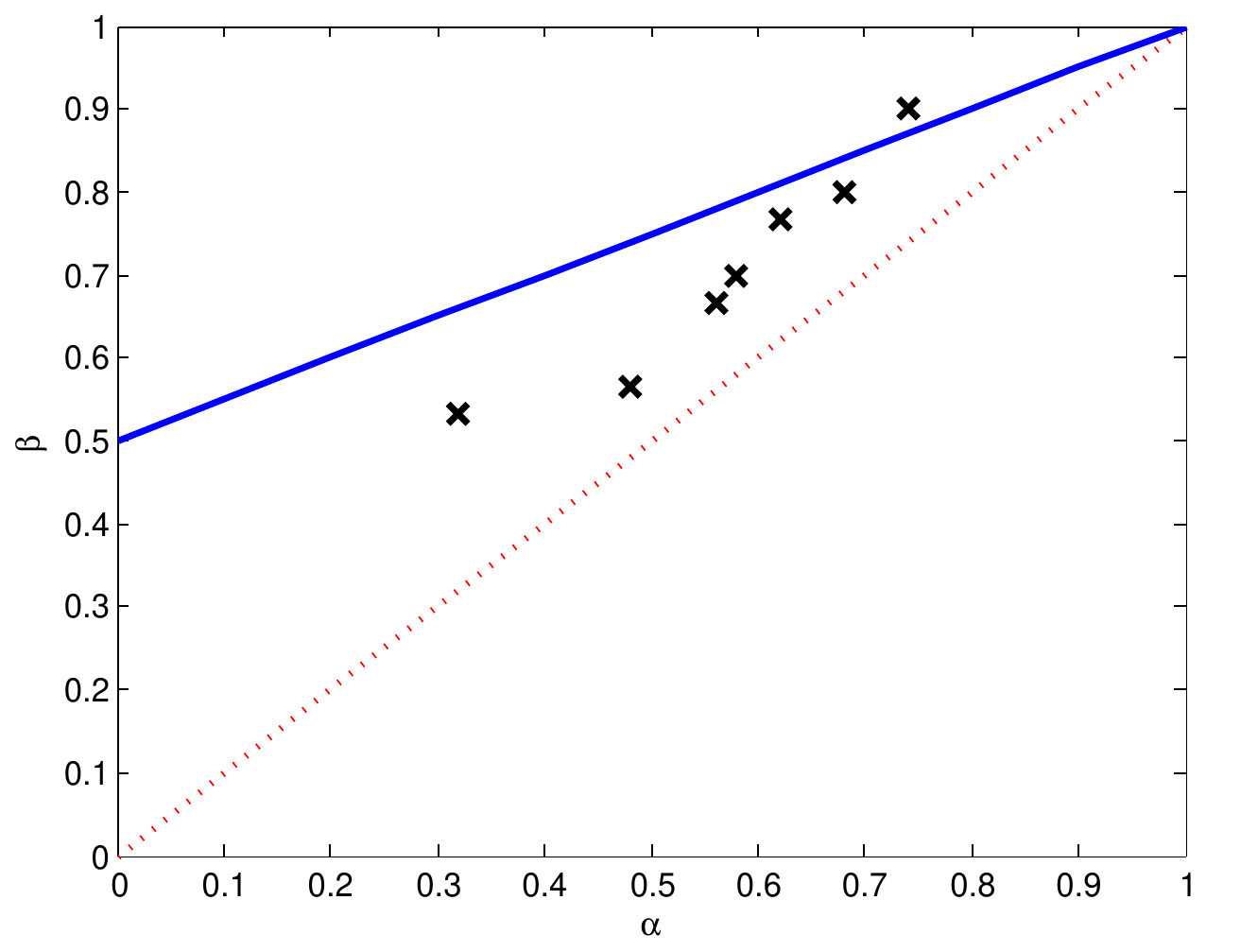}{3.2in}
\centering
\caption{Simulation result of the spectral method given in Algorithm \ref{alg:spectralmethod} with $n=1000$ and $p=0.05$. The $x$-axis corresponds to erasure probability $\epsilon = 1-n^{-\alpha}$ and $y$-axis
corresponds to cluster size $K= n^\beta$. Each data point in the plot indicates the maximum value of $\alpha$ for which the spectral method succeeds with a given $\beta$.  The blue solid line shows the performance of the spectral method predicted by Theorem \ref{ThmSpectralMethod}. The red dotted line shows the performance of the convex method predicted by Theorem \ref{result:convex_relaxation}.}
\label{fig:spectral_simulation}
\end{figure}


\section{Concluding Remarks} \label{Sec:Conclusion}
This paper studies the problem of inferring hidden row and column clusters of binary matrices from a few noisy observations through theoretical analysis and numerical experiments.
More extensive simulation results will be presented in a longer version of the paper. Several future directions are of interest. First, proving Conjecture \ref{result:conjecture} is important to fully understand the performance of the convex method. Second, a tight performance analysis of the ML estimation~\eqref{MLE} is needed to achieve the lower bound. Third, it is interesting to extend our analysis to block rating matrices having real-valued entries.

\bibliographystyle{IEEEtran}
\bibliography{BibCDRecommender}  

\end{document}